\documentclass[10pt,journal,twoside]{IEEEtran}

\usepackage[T1]{fontenc}
\usepackage[utf8]{inputenc}
\usepackage[margin=1in]{geometry}
\usepackage{graphicx}
\usepackage{booktabs}
\usepackage{multirow}
\usepackage{algorithm}
\usepackage{cite}
\usepackage{hyperref}
\usepackage{microtype}
\usepackage{caption}
\usepackage{float}
\usepackage{siunitx}
\usepackage{amsmath, amssymb, amsthm}
\usepackage{bm}
\usepackage{algpseudocode}

\newtheorem{lemma}{Lemma}

\hypersetup{
  colorlinks=true,
  linkcolor=blue,
  citecolor=blue,
  urlcolor=blue
}

\title{Class--Confidence Aware Reweighting for Long-Tailed Learning}

\author{Brainard~Philemon~Jagati%
       \thanks{Corresponding author  E-mail: jtembhurne@iiitn.ac.in}, 
       \and Jitendra~Tembhurne%
       ,
       \and Harsh~Goud%
       ,
       \and Rudra~Pratap~Singh%
       ,
       \and Chandrashekhar~Meshram%
\thanks{The authors are with the Department of Computer Science \& Engineering, Electronics \& Communication Engineering,  
Indian Institute of Information Technology Nagpur 441108, India, and Jayawanti Haksar Govt. Post Graduate College, Betul, MP, India.}}

\begin{document}

\maketitle

\begin{abstract}
Deep neural network models degrade significantly in the long-tailed data distribution, with the overall training data dominated by a small set of classes in the head, and the tail classes obtaining less training examples. Addressing the imbalance in the classes, attention in the related literature was given mainly to the adjustments carried out in the decision space in terms of either corrections performed at the logit level in order to compensate class-prior bias, with the least attention to the optimization process resulting from the adjustments introduced through the differences in the confidences among the samples.
In the current study, we present the design of a class and confidence-aware re-weighting scheme for long-tailed learning. This scheme is purely based upon the loss level and has a complementary nature to the existing methods performing the adjustment of the logits. In the practical implementation stage of the proposed scheme, we use an $\Omega(p_t,f_c)$ function. This function enables the modulation of the contribution towards the training task based upon the confidence value of the prediction, as well as the relative frequency of the corresponding class.
Our observations in the experiments are corroborated by significant experimental results performed on the CIFAR-100-LT, ImageNet-LT, and iNaturalist2018 datasets under various values of imbalance factors that clearly authenticate the theoretical discussions above.
\end{abstract}


\section{Introduction}
\IEEEPARstart{D}{eep} learning models derived from neural network architectures \cite{Goodfellow-et-al-2016} have become the most popular choice for a broad spectrum of classification tasks. These models include speech recognition \cite{speech}, natural language processing \cite{torfi2021naturallanguageprocessingadvancements}, anomaly detection \cite{anomaly}, recommendation systems \cite{recommendation}, and health care \cite{healthcare}. In all such systems, a classifier and an encoder transform the hierarchical representation of a neural network into a discrete class, making accuracy a crucial determinant of system performance. This explains why a great deal of effort and innovation have been directed toward developing optimized architectures for greater accuracy, precision, and generalization.

Despite these advances, most real-world datasets are not collected under controlled or balanced conditions. Instead, data are often acquired from heterogeneous and naturally occurring sources, where certain classes appear far more frequently than others. This leads to imbalanced or long-tailed data distributions \cite{imbalanced-dataset,imbalanced2}, under which neural networks tend to overfit dominant (head) classes while failing to learn discriminative representations and reliable decision boundaries for under-represented (tail) classes.

To address class imbalance, a variety of strategies have been explored in prior work. A widely adopted approach is data re-sampling \cite{resample}, where minority classes are over-sampled or majority classes are under-sampled to improve class balance \cite{Chawla_2002,Rezvani_2024}. Though successful in some applications, re-sampling may end up replicating some samples or removing useful information in the process. A related but distinct approach is where the model addresses loss re-weighting based on classes. This involves assigning higher weights to the losses of under-represented classes. One such technique is the Class-Balanced Loss \cite {cui2019classbalancedlossbasedeffective}, which reweights the gradients using effective sample sizes. This method takes advantage of statistics but addresses the issue only at the class level.
In parallel, confidence-based methods such as Focal Loss \cite{lin2018focallossdenseobject} prioritize hard-to-classify samples by down-weighting well-classified ones; however, these approaches are class-agnostic and do not explicitly model long-tailed class distributions. More recently, logit-level correction and margin-based methods have been proposed to directly counteract class-prior bias in the decision space \cite{menon2021longtaillearninglogitadjustment,zhang2024gradientawarelogitadjustmentloss,ren2020balancedmetasoftmaxlongtailedvisual}. Although effective, these methods typically apply uniform, class-level corrections and do not explicitly account for sample-wise confidence variations during optimization. In contrast, our work focuses on loss-level reweighting that jointly incorporates class frequency and prediction confidence, offering a complementary mechanism to decision-space corrections by adaptively modulating gradient contributions without modifying logits or inference behavior.
%\medskip \\
\begin{itemize}
\item To propose a simple and lightweight class–confidence aware reweighting formulation that modulates sample-wise optimization dynamics without modifying logits, margins, or inference behavior. The formulation introduces a single suppression parameter $\omega$, which controls the degree of hardness modulation in a stable and interpretable manner.
\item To propose a unified weight function that multiplies the cross-entropy loss with an exponential term that involves the predicted class probability $p_t.$ of the positive class and its frequency $f_c.$ The proposed framework seeks to emphasize stronger gradients of samples from minority classes with weak confidence levels while suppressing the gradients of samples with strong confidence levels in the major classes.
\item To evaluate the effectiveness of our method through extensive experiments on common long-tailed datasets: CIFAR-100-LT~\cite{cifar100}, ImageNet-LT~\cite{imagenet-lt}, and iNaturalist~\cite{vanhorn2018inaturalistspeciesclassificationdetection}, on various values of the imbalance factor. Our reweighting method is shown to improve the accuracy of tail classes while maintaining competitive performance on head classes compared to the recent methods.
\end{itemize}

\section{Related Work}
Current methods for solving the class imbalance problem in supervised learning can be generally classified into sampling methods and loss-weighting approaches.

\subsection{Data Re-sampling}

Techniques of re-sampling data attempt to counterbalance class distributions either through over-sampling in the minority class or through under-sampling in the majority class \cite{Chawla_2002, Rezvani_2024}. Although such techniques may have a positive impact on the imbalance at the data level, certain side effects inevitably crop up. Simple over-sampling may lead to repeated samples, thereby preferring model overfitting to the lesser variation in the tail class, as opposed to learning more generalizable features \cite{Du_2024}. Under-sampling would necessarily entail losing informative samples in the major class, thereby suppressing the diversity of the data being learned. Recent works further accentuate the fact that re-sampling significantly changes over the empirical distribution of the data, causing a mismatch between the training and testing distributions \cite{Ma_2025}. Such a mismatch in distributions might affect the quality of the feature space, even if the classifier boundaries become more balanced. To overcome some of these deficiencies, recent methods aimed to decouple balancing classifier learning from representation learning or used progressive mixing schemes \cite{Li_2025} that used tail class information without vitiating the entire feature space.

\subsection{Class-aware Loss Reweighting}
In contrast to logit adjustments that adjust the boundary via model output offsets, class-aware loss weighting proposes to reduce class imbalance issues through direct adjustments to the penalty weight of each class during training. The traditional method tends to employ fixed weight schemes inversely proportional to class frequency or using the effective number of examples of each class \cite{cui2019classbalancedlossbasedeffective}, presuming each class's loss weight should be inversely proportional to its size.

However, recent works call this assumption into question by claiming that sample cardinality is a rough guideline that cannot capture the semantic or geometric variations among classes satisfactorily. Chen et al.\cite{Chen_2025_ADIR} demonstrate that tail classes tend to have an even smaller feature volume than head classes and present the Adaptive Diversity-Induced Reweighting (ADIR) method that adjusts the weights of the loss function in accordance with the natural diversity of classes rather than their sample number. To move beyond the task of static sample reweighting, it has been noticed that the seesaw effect can be caused if long-tailed recognition is treated as an optimization problem with a single objective.

For the purpose of mitigating this trade-off, the recent literature turned to the area of multi-objective optimization (MOO). As illustrated by He et al. \cite{He_2024_Pareto}, this can be achieved using a Pareto optimization approach that focuses on learning the class-weight parameters, dynamically exploring the descent directions to minimize the conflicts in the gradients of the head and tail classes. Another example of this category is the MOOSF approach \cite{Peng_2024_MOOSF}, where the strategy fusion concept can be used for the purpose of adaptive class-weight learning to keep the learning trajectory Pareto optimal.

\subsection{Hardness-aware Weighting}
The other line of work involves the use of techniques that highlight challenging instances in the training process through the adjustment of loss based on prediction confidences. Notable works in this line include the seminal work of Focal Loss \cite{lin2018focallossdenseobject}, where the loss of well-classified instances is suppressed in favor of the loss of the challenging instances. However, in the context of the long-tailed problem, the suppression of low-confidence predictions without distinction can rather contribute to the amplification of outliers \cite{Li_2025_Uncertainty}.

To overcome these limitations, recent approaches have moved beyond simple probability-based weighting. For example, gradient-norm–balanced methods \cite{Wang_2024_Gradient} adjust sample weights based on the magnitude of backpropagated gradients rather than output probabilities alone, ensuring that hard samples meaningfully contribute to representation learning without destabilizing the classifier. Similarly, uncertainty-guided frameworks \cite{Zhang_2025_Curriculum} adopt dynamic curriculum learning strategies to distinguish between \textit{useful hard} samples and \textit{harmful noisy} ones, gradually increasing emphasis on tail classes only after the model has learned robust representations for head classes.

Despite their effectiveness, these methods often rely on complex multi-stage training pipelines or auxiliary uncertainty estimation modules. In contrast, our approach achieves joint class–confidence modulation through a simple loss-level reweighting term that integrates seamlessly with existing logit-adjustment objectives, enabling single-stage training without additional architectural components.

\subsection{Logit-level Adjustment}
Recent advances in long-tailed learning have increasingly shifted attention from data-level re-sampling toward directly correcting class imbalance in the decision space through logit-level adjustment. Seminal work on Logit Adjustment (LA) \cite{menon2021longtaillearninglogitadjustment} introduced a statistically principled mechanism that shifts prediction scores according to class priors, achieving Fisher consistency for balanced error minimization.

Building upon this foundation, subsequent studies have proposed more dynamic and geometry-aware refinements. For example, \textit{Generalized Logit Adjustment} (GLA) \cite{Zhang_2024_GLA} extends boundary calibration to account for mismatches between training and testing distributions, while Gaussian-based adjustment methods \cite{Wang_2024_Gaussian} perturb logits during training to emulate a more balanced feature-space geometry. In a similar vein, adaptive margin techniques \cite{Son_2025_DBM} adjust decision boundaries based on estimated class difficulty rather than fixed frequency statistics. Despite their empirical effectiveness, these approaches typically apply uniform, class-level corrections or rely on post-hoc bias terms, which may not optimally influence representation learning dynamics, particularly during the early stages of training.

In contrast to these decision-space modifications, our work focuses on \textit{loss-level reweighting} that jointly incorporates class frequency and sample-wise confidence into a single modulation term. This design is inherently complementary to logit-adjustment methods: while LA corrects the \textit{decision boundary} at inference, our approach shapes the \textit{optimization trajectory} by adaptively modulating gradients during training. As a result, the proposed method can be seamlessly integrated with existing logit-adjusted objectives without requiring changes to model architecture, margin design, or inference procedures.

\section{Methodology}

\subsection{Problem Formulation}
We consider a long-tailed classification dataset
$\mathcal{D} = \{(x_i, y_i)\}_{i=1}^{N}$ comprising $K$ classes.
Let $N_c$ denote the number of samples belonging to class $c$, and let
$f_c = \frac{N_c}{N}$ represent its empirical class frequency.
Given an input $x$, a neural network outputs a logit vector
$\mathbf{z} \in \mathbb{R}^K$, which is converted into predicted class
probabilities $\mathbf{p} = \mathrm{softmax}(\mathbf{z})$.
We denote by $p_t$ the predicted probability assigned to the ground-truth
class $t$.

Our formulation builds upon a base loss
$\mathcal{L}_{\mathrm{base}}(p_t)$, such as the standard cross-entropy loss
or its logit-adjusted variants.
To account for both class imbalance and sample difficulty, we introduce
a sample-wise reweighting function $\Omega(p_t, f_c)$ and define the total
training objective as,
\begin{equation}
\mathcal{L}_{\mathrm{total}}(p_t, f_c)
= \Omega(p_t, f_c)\,\mathcal{L}_{\mathrm{base}}(p_t).
\end{equation}

Na\"ive importance-weighting schemes, such as $(f_c p_t)^{-1}$, suffer from
unbounded variance as $p_t \rightarrow 0$, which can result in unstable
gradients and hinder convergence. Our objective is therefore to design a
weighting function that is bounded, continuous, and differentiable, and
that jointly depends on class frequency and prediction confidence while
preserving stable optimization dynamics.

\subsection{Motivation}
Let $p_t \in [0,1]$ denote the prediction confidence associated with the
ground-truth class. We introduce a margin-like function
\begin{equation}
    m(p_t) = \omega - p_t,
\end{equation}
where $\omega \in (0,1]$ serves as a user-defined pivot that separates
low-confidence predictions from high-confidence ones. This formulation
encodes the intuition that samples with confidence below $\omega$ should
receive stronger emphasis during optimization, while overly confident
samples should be softly suppressed.

We derive the weighting function from a maximum-entropy perspective.
Following the Maximum Entropy Principle \cite{jaynes1957information}, we
seek a density function $w(p_t)$ that maximizes the expected margin utility
while remaining close to a uniform prior $u$, measured via the
Kullback--Leibler (KL) divergence \cite{kullback}. Formally, we consider the optimization
problem
\begin{equation}
    \max_{w \in \mathcal{W}} \; \mathcal{J}(w)
    = \mathbb{E}_w[m(p_t)]
    - \frac{1}{\beta_c}\,\mathrm{KL}(w \,\|\, u),
\end{equation}
subject to the normalization constraint
$\int_0^1 w(p_t)\, dp_t = 1$.
Here, $\beta_c > 0$ is a class-dependent information capacity parameter
that controls the trade-off between margin maximization and entropy
regularization. Since the prior $u$ is uniform on $[0,1]$, the KL term
reduces to the negative entropy
$\int_0^1 w(p_t)\ln w(p_t)\, dp_t$.

To solve this variational problem, we apply the method of Lagrange
multipliers \cite{lagrangian} and construct the Lagrangian functional
$\mathcal{L}(w, \lambda)$:
\begin{equation}
\begin{aligned}
\mathcal{L}(w, \lambda)
&= \int_0^1 w(p_t)(\omega - p_t)\, dp_t \\
&\quad - \frac{1}{\beta_c} \int_0^1 w(p_t)\ln w(p_t)\, dp_t \\
&\quad - \lambda \left( \int_0^1 w(p_t)\, dp_t - 1 \right).
\end{aligned}
\end{equation}

Taking the functional derivative with respect to $w(p_t)$ and setting it to
zero yields the stationary condition
\begin{equation}
    \frac{\delta \mathcal{L}}{\delta w(p_t)}
    = (\omega - p_t)
    - \frac{1}{\beta_c}\big(1 + \ln w(p_t)\big)
    - \lambda
    = 0.
\end{equation}
Rearranging terms to isolate $\ln w(p_t)$ gives
\begin{equation}
\begin{aligned}
    \frac{1}{\beta_c}\ln w(p_t)
    &= (\omega - p_t) - \lambda - \frac{1}{\beta_c}, \\
    \ln w(p_t)
    &= \beta_c(\omega - p_t) - (1 + \beta_c \lambda).
\end{aligned}
\end{equation}
Exponentiating both sides leads to an exponential-family solution
\begin{equation}
    w^*(p_t)
    = \exp\!\big(\beta_c(\omega - p_t)\big)
      \cdot \exp\!\big(-(1 + \beta_c \lambda)\big).
\end{equation}
The second term is a constant independent of $p_t$ and can be absorbed
into a normalization factor $Z(\beta_c)$. The optimal density therefore
takes the form
\begin{equation}
    w^*(p_t)
    = \frac{1}{Z(\beta_c)}
      \exp\!\big(\beta_c(\omega - p_t)\big).
\end{equation}

For the purpose of loss reweighting, the normalization constant does not
affect gradient directions and can be omitted, yielding the final
per-sample weighting function
\begin{equation}
    \Omega(p_t; \beta_c)
    \propto \exp\!\big(\beta_c(\omega - p_t)\big).
\end{equation}

\subsection{Dual-Phase Class--Confidence Coupling}
\label{sec:dual_phase}
The exponential margin-based weighting $\exp(\omega - p_t)$, derived from the
maximum-entropy formulation, naturally emphasizes low-confidence predictions.
However, this formulation is agnostic to class frequency and therefore treats
rare and frequent classes symmetrically. Such behavior is suboptimal in
long-tailed settings. To explicitly account for class imbalance, we introduce a
\emph{dual-phase} coupling mechanism that modulates the information capacity
parameter $\beta_c$ based on the prediction confidence regime.

\paragraph{Asymmetric Requirements}
Effective optimization under long-tailed distributions demands asymmetric
behavior across different confidence regimes:
\begin{itemize}
    \item \textbf{Exploration Phase ($p_t < \omega$):} When the model is uncertain,
    gradients should be \emph{amplified}, particularly for rare classes, to
    encourage representation learning and accelerate convergence.
    \item \textbf{Consolidation Phase ($p_t \ge \omega$):} When predictions are
    confident, gradients should be \emph{suppressed}, especially for dominant
    head classes that would otherwise overwhelm the loss and lead to overfitting.
\end{itemize}
To satisfy the requirements within a single unified function, we
define a \emph{Dual-Phase Modulated Frequency} $f_c'(p_t)$:
\begin{equation}
    f_c'(p_t) =
    \begin{cases}
        f_c, & p_t < \omega \quad \text{(Standard Frequency)}, \\
        1 - f_c, & p_t \ge \omega \quad \text{(Inverted Frequency)}.
    \end{cases}
    \label{eq:dual_phase_freq}
\end{equation}
This inversion has a very important function in gradient dynamics. As illustrated in
the gradient analysis (Figure~\ref{fig:grad_analysis}),
regime where the exponent $(\omega - p_t)$ is negative, a larger base
results in the production of a smaller effective weight, paving the way for a strong gradient.
suppression. By inverting the frequency ($f_c \rightarrow 1 - f_c$), frequent
classes ($f_c \approx 1$) are assigned a larger effective base
($e - (1 - f_c) \approx e$) than rare classes, ensuring that head-class gradients
are suppressed more aggressively during consolidation.

To realize this asymmetric behavior, we promote the information capacity
parameter from a static constant $\beta_c$ to a confidence-dependent function
$\beta_c(p_t)$. Specifically, we define the adaptive capacity as,
\begin{equation}
    \beta_c(p_t) = \ln\!\big(e - f_c'(p_t)\big).
\end{equation}
Substituting this adaptive capacity into the exponential-family solution yields
the final \emph{Class--Confidence Aware Reweighting (CCAR)} function:
\begin{equation}
    \Omega(p_t, f_c) = \big(e - f_c'(p_t)\big)^{\omega - p_t}.
    \label{eq:final_weight_coupling}
\end{equation}
\begin{figure}[H]
    \centering
    \includegraphics[width=0.9\linewidth]{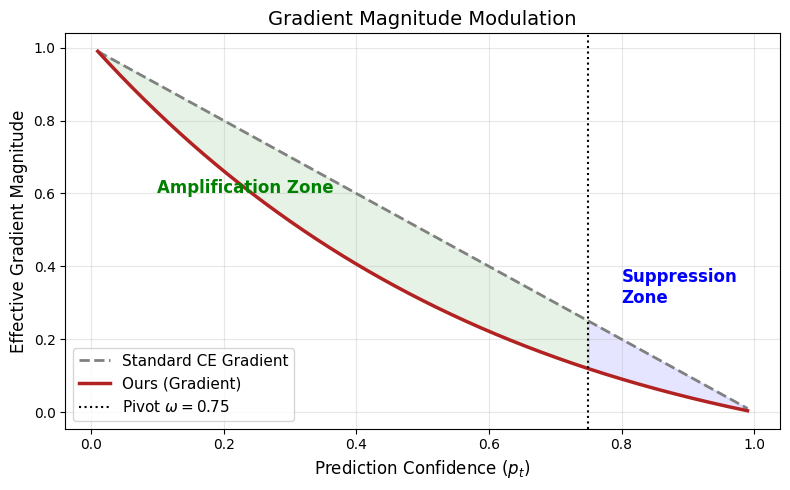}
    \caption{\textbf{Gradient Modulation Analysis.} Comparison of effective
    gradient magnitudes. The proposed method (red) amplifies gradients in the
    low-confidence region ($p_t < \omega$) to accelerate learning for hard
    samples, while suppressing gradients in the high-confidence region
    ($p_t > \omega$) to prevent overfitting to easy head-class samples.}
    \label{fig:grad_analysis}
\end{figure}
\paragraph{Theoretical Consistency}
Despite the piecewise definition of the dual-phase frequency term $f_c'$, the
resulting weighting function $\Omega(p_t, f_c)$ remains continuous across the
entire confidence domain. In particular, at the pivot point $p_t = \omega$, the
exponent $(\omega - p_t)$ vanishes, yielding
\begin{equation}
    \lim_{p_t \to \omega} \Omega(p_t, f_c)
    = (e - f_c')^{0}
    = 1.
\end{equation}
This property guarantees a smooth transition between the amplification and
suppression regimes, independent of the underlying class frequency.

\subsection{Continuity and Derivative Behavior}
By construction, the weighting function satisfies
\begin{equation}
    \Omega(\omega, f_c) = 1 \quad \forall\, c,
\end{equation}
which establishes continuity at the confidence pivot.

The derivative of $\Omega$ with respect to the confidence variable $p_t$ is given
by
\begin{equation}
    \frac{\partial \Omega}{\partial p_t}
    = -\ln\!\big(e - f_c'(p_t)\big)\,\Omega(p_t, f_c).
\end{equation}
This expression shows that the gradient magnitude is directly modulated by the
effective information capacity term $\ln(e - f_c'(p_t))$, linking confidence,
class frequency, and gradient scaling in a unified manner.

Evaluating the one-sided derivatives at the pivot $p_t = \omega$ yields
\begin{equation}
\lim_{p_t \to \omega^-} \nabla \Omega
= -\ln(e - f_c), \qquad
\lim_{p_t \to \omega^+} \nabla \Omega
= -\ln\!\big(e - (1 - f_c)\big).
\end{equation}
Although the derivative exhibits a finite jump due to the phase transition, this
discontinuity is uniformly bounded:
\begin{equation}
    \left|
    \ln\!\frac{e - f_c}{e - 1 + f_c}
    \right|
    \le \ln\!\frac{e}{e - 1}
    \approx 0.46,
    \label{eq:derivative_jump}
\end{equation}
which ensures Lipschitz continuity with controlled non-smoothness. Consequently,
the proposed weighting function preserves stable gradient behavior and does not
introduce pathological optimization artifacts near the confidence pivot.

\subsection{Gradient Dynamics Analysis}
\label{sec:gradient_analysis}
For notational simplicity, we define
$\gamma_c = e - f_c'(p_t)$.
We assume the base loss to be the standard Cross-Entropy,
$\mathcal{L}_{\text{base}} = -\ln p_t$.
Since the dual-phase frequency term $f_c'(p_t)$ is piecewise constant
(Eq.~\ref{eq:dual_phase_freq}), its derivative with respect to $p_t$ is
zero almost everywhere (i.e., for all $p_t \neq \omega$).
Accordingly, $\gamma_c$ can be treated as locally constant during
differentiation.

\begin{lemma}[Gradient Derivation]
The gradient of the CCAR loss with respect to the logit vector
$\mathbf{z}$ takes the form
\begin{equation}
    \nabla_{\mathbf{z}} \mathcal{L}_{\text{total}}
    = \Psi(p_t, \gamma_c)\,(\mathbf{p} - \mathbf{e}_t),
    \label{eq:gradient_full}
\end{equation}
where $\mathbf{e}_t$ denotes the one-hot encoding of the target class, and
the scalar \emph{modulation factor} $\Psi$ is given by
\begin{equation}
    \Psi(p_t, \gamma_c)
    = \gamma_c^{\omega - p_t}
      \left( 1 - p_t \ln \gamma_c \ln p_t \right).
\end{equation}
\end{lemma}

\begin{proof}
We apply the chain rule,
$\nabla_{\mathbf{z}} \mathcal{L}
= \frac{\partial \mathcal{L}}{\partial p_t}
  \nabla_{\mathbf{z}} p_t$.

\textbf{Step 1: Derivative with respect to $p_t$.}
We decompose the total loss into a weighting term and a base loss term:
\begin{align*}
    u(p_t) &= \gamma_c^{\omega - p_t}, \\
    v(p_t) &= -\ln p_t .
\end{align*}
Using the product rule,
$\frac{\partial \mathcal{L}}{\partial p_t} = u'v + uv'$:
\begin{enumerate}
    \item The derivative of the weighting term is
    \begin{equation}
        u'(p_t)
        = \frac{d}{dp_t}
          \exp\!\big((\omega - p_t)\ln \gamma_c\big)
        = -\ln \gamma_c \cdot \gamma_c^{\omega - p_t}.
    \end{equation}
    \item The derivative of the Cross-Entropy term is
    \begin{equation}
        v'(p_t) = -\frac{1}{p_t}.
    \end{equation}
\end{enumerate}
Substituting these expressions yields
\begin{align}
    \frac{\partial \mathcal{L}}{\partial p_t}
    &= \left(-\ln \gamma_c \cdot \gamma_c^{\omega - p_t}\right)(-\ln p_t)
       + \gamma_c^{\omega - p_t}\left(-\frac{1}{p_t}\right) \nonumber \\
    &= \gamma_c^{\omega - p_t} \ln \gamma_c \ln p_t
       - \frac{\gamma_c^{\omega - p_t}}{p_t} \nonumber \\
    &= \frac{\gamma_c^{\omega - p_t}}{p_t}
       \left(p_t \ln \gamma_c \ln p_t - 1\right).
    \label{eq:grad_wrt_p}
\end{align}

\textbf{Step 2: Derivative with respect to logits $\mathbf{z}$.}
The derivative of the softmax probability for the target class is
\begin{equation}
    \nabla_{\mathbf{z}} p_t
    = p_t(\mathbf{e}_t - \mathbf{p}).
\end{equation}
Combining this with Eq.~\eqref{eq:grad_wrt_p} gives
\begin{align}
    \nabla_{\mathbf{z}} \mathcal{L}
    &= \left[
        \frac{\gamma_c^{\omega - p_t}}{p_t}
        (p_t \ln \gamma_c \ln p_t - 1)
       \right]
       \left[
        p_t(\mathbf{e}_t - \mathbf{p})
       \right] \nonumber \\
    &= \gamma_c^{\omega - p_t}
       (p_t \ln \gamma_c \ln p_t - 1)
       (\mathbf{e}_t - \mathbf{p}).
\end{align}
Rewriting the expression in the conventional gradient direction
$(\mathbf{p} - \mathbf{e}_t)$ yields
\begin{equation}
    \nabla_{\mathbf{z}} \mathcal{L}
    = \gamma_c^{\omega - p_t}
      \left(1 - p_t \ln \gamma_c \ln p_t\right)
      (\mathbf{p} - \mathbf{e}_t).
\end{equation}
\end{proof}

\paragraph{Self-Regulating Properties}
The derived expression for the gradient indicates two significant stability
properties:
\begin{itemize}
\item \textbf{Boundedness}
The function $x \ln x$ is bounded on the interval $[0,1]$, with a minimum at $-1/e$
Therefore, the modulation factor $\Psi(p_t, \gamma_c)$
finite for all $p_t \in (0,1]$, to prevent gradient explosion even in
rare classes.
\item \textbf{Entropy Awareness}
The expression $-p_t \ln p_t$ stands for the entropy of prediction,
Adaptively scaling gradients based on the uncertainty. This mechanism*
that prioritize informative, uncertain samples over just random ones
with large loss values.
\end{itemize}

\begin{figure}[H]
    \centering
    \includegraphics[width=0.8\linewidth]{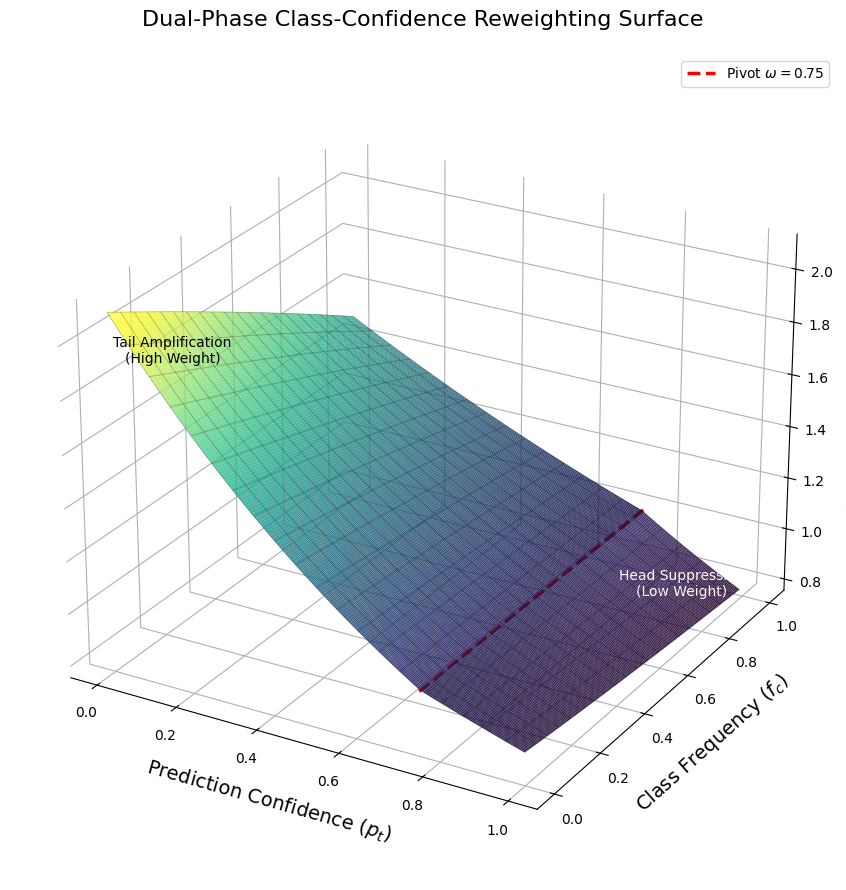}
    \caption{
   Three-dimensional wireframe plot of weighting function
$\Omega(p_t,f_c)$. The above surface shows how the weight varies with
prediction confidence $p_t$ for varying class frequencies $f_c$.
The pivot point $\omega$ helps distinguish between the low and high confidence regions, indicating
asymmetric pattern for tail and head classes with smooth
transitions.
    }
    \label{fig:omega_3d}
\end{figure}

\section{Experiments}

\subsection{Datasets}
We test the proposed method on three popular benchmarks for
long-tailed visual recognition: CIFAR-100-LT \cite{cifar100},
ImageNet-LT \cite{imagenet-lt}, and iNaturalist2018
\cite{vanhorn2018inaturalistspeciesclassificationdetection}. It ranges both synthetically created and naturally occurring long-tailed
distributions, and it handles small- or large-scale classification problems.
Collectively, the above functions form a holistic framework for testing
under various levels of class imbalance.

Summary statistics for the datasets are shown in
Table~\ref{tab:datasets}.

\begin{table}[H]
\centering
\caption{Statistics of the long-tailed datasets used in our experiments.}
\label{tab:datasets}
\begin{tabular}{lccc}
\toprule
Dataset & Classes & Train & Test \\
\midrule
CIFAR-100-LT    & 100   & 50,000  & 10,000 \\
ImageNet-LT    & 1,000 & 115,846 & 50,000 \\
iNaturalist2018& 8,142 & 437,513 & 24,426 \\
\bottomrule
\end{tabular}
\end{table}

\subsection{Implementation Details}
All experiments are implemented using PyTorch.
For CIFAR-100-LT, we follow the same training protocol, network
architecture, data preprocessing, and optimization settings as the
Balanced Softmax framework \cite{ren2020balancedmetasoftmaxlongtailedvisual}.
For ImageNet-LT and iNaturalist2018, we strictly adhere to the experimental
setup described in \cite{zhang2022longtailedclassificationgradualbalanced}.

In all settings, the proposed reweighting term is applied on top of the
corresponding logit-adjusted loss (e.g., Logit Adjustment or Balanced
Softmax) without modifying the underlying model architecture, training
schedule, or evaluation procedure. This design ensures a fair and
controlled comparison by isolating the effect of the proposed weighting
strategy. Unless otherwise specified, we set the confidence pivot
$\omega = 0.75$ and employ the adaptive capacity
$\beta_c(p_t)$ defined in Eq.~\ref{eq:final_weight_coupling}.

\subsection{Experimental Results}
We fix the suppression factor to $\omega=0.75$ in all experiments as we mentioned in Section 5.4, as this value consistently showed the best or near-best performance across CIFAR-100-LT, ImageNet-LT, and iNaturalist2018 in our preliminary validation.

\textbf{CIFAR-100-LT.}
As reported in Table~\ref{tab:cifar_results}, the proposed method consistently
improves performance across all imbalance factors.
When applied to vanilla cross-entropy, our reweighting strategy increases
Top-1 accuracy from 34.84\% to 36.12\% at IF=200, from 38.43\% to 39.86\% at
IF=100, and from 43.90\% to 45.25\% at IF=50.
When combined with logit-adjustment techniques, the proposed approach yields
further gains. In particular, integrating our method with Balanced Softmax
(Ours + BS) achieves 46.20\% and 48.91\% under IF=200 and IF=100, respectively,
substantially outperforming the corresponding Balanced Softmax baselines
(43.30\% and 45.00\%).
Notably, compared to the margin-based LDAM-DRW method
\cite{cao2019learningimbalanceddatasetslabeldistributionaware}, our approach
achieves superior performance (e.g., 52.78\% vs.\ 47.62\% at IF=50), while
retaining a simpler single-stage training pipeline.

\textbf{ImageNet-LT.}
In ImageNet-LT with a ResNet-50 backbone
(Table~\ref{tab:imagenet_lt_complete_results}), supplementing cross-entropy with the proposed suppression mechanism improves Top-1 accuracy from 41.60\% to
45.62\%, which reflects a gain of +4.02\%.
Combining the proposed method with Logit Adjustment gives
consistent improvements: Ours + LA achieves 52.76\% compared to 51.10\% for
vanilla Logit Adjustment \cite{menon2021longtaillearninglogitadjustment}.
These findings here indicate that, as pertaining to this, dominant gradients are explicitly regulated in nature.
training yields complementarities beyond prior logit-shifting.
strategies by themselves.

\textbf{iNaturalist2018.}
As it can be seen from the Table~\ref{tab:inat_results} and shown in
Proposed method. Figure~\ref{fig:acc_breakdown} shows that
improvements on all frequency subsets of iNaturalist2018.
Using cross-entropy as the base loss, our approach improves overall accuracy.
from 61.70\% to 64.90\%, with gains observed across the \textit{Many}
(72.20\% $\rightarrow$ 72.68\%), $\textit{Medium}$
63.00\% → 63.91\%, and Few
(57.20\% → 58.12\%) categories.
When combined with Balanced Softmax, the proposed method yields an overall
it achieves an accuracy of 70.10\%, which is higher than using only Balanced Softmax at 69.80\%
\cite{ren2020balancedmetasoftmaxlongtailedvisual}
These results reveal an improved decision boundary allocation for rare
classes in fine-grained and naturally long-tailed recognition settings.

% ---------------- CIFAR-100-LT Table ----------------
\begin{table}[H]
\centering
\caption{Overall Top-1 accuracy (\%) on CIFAR-100-LT with different imbalance factors.}
\label{tab:cifar_results}
\begin{tabular}{lccc}
\toprule
Method & IF=200 & IF=100 & IF=50 \\
\midrule
\multicolumn{4}{l}{\textit{Standard / Re-weighting}} \\
Cross-Entropy~\cite{mao2023crossentropylossfunctionstheoretical} & 34.84 & 38.43 & 43.90 \\
Focal Loss~\cite{lin2018focallossdenseobject} & 35.60 & 38.40 & 44.30 \\
Class-Balanced~\cite{cui2019classbalancedlossbasedeffective} & -- & 39.60 & 45.17 \\
Ours + CE & \textbf{36.12} & \textbf{39.86} & \textbf{45.25} \\
\midrule
\multicolumn{4}{l}{\textit{Margin-based}} \\
LDAM-DRW~\cite{cao2019learningimbalanceddatasetslabeldistributionaware} & 38.91 & 42.04 & 47.62 \\
\midrule
\multicolumn{4}{l}{\textit{Logit-Adjustment}} \\
Logit Adjust~\cite{menon2021longtaillearninglogitadjustment} & -- & 43.90 & 52.30 \\
Balanced Softmax~\cite{ren2020balancedmetasoftmaxlongtailedvisual} & 43.30 & 45.00 & 50.50 \\
Focal + LA & -- & 43.94 & 52.01 \\
Focal + BS & 43.36 & 45.11 & 49.98 \\
Ours + LA & \textbf{45.93} & \textbf{44.65} & \textbf{52.78} \\
Ours + BS & \textbf{46.20} & \textbf{48.91} & \textbf{51.94} \\
\midrule
\multicolumn{4}{l}{\textit{Decoupled / Two-stage}} \\
cRT + mixup~\cite{kang2020decouplingrepresentationclassifierlongtailed} & 41.73 & 45.12 & 50.86 \\
PaCo~\cite{cui2021parametriccontrastivelearning} & -- & 52.00 & 56.00 \\
\bottomrule
\end{tabular}
\end{table}

% ---------------- ImageNet-LT Table ----------------
\begin{table}[H]
\centering
\caption{Overall Top-1 accuracy (\%) on ImageNet-LT using ResNet-50}
\label{tab:imagenet_lt_complete_results}
\begin{tabular}{lc}
\toprule
Method & Overall \\
\midrule
\multicolumn{2}{l}{\textit{Standard / Re-weighting}} \\
Cross-Entropy~\cite{mao2023crossentropylossfunctionstheoretical} & 41.60 \\
CB-Focal~\cite{lin2018focallossdenseobject} & 43.90 \\
Ours + CE & \textbf{45.62} \\
\midrule
\multicolumn{2}{l}{\textit{Margin-based}} \\
LDAM-DRW~\cite{cao2019learningimbalanceddatasetslabeldistributionaware} & 48.80 \\
\midrule
\multicolumn{2}{l}{\textit{Logit-Adjustment}} \\
Logit Adjust~\cite{menon2021longtaillearninglogitadjustment} & 51.10 \\
Ours + LA & \textbf{52.76} \\
\midrule
\multicolumn{2}{l}{\textit{Decoupled / Two-stage}} \\
cRT~\cite{kang2020decouplingrepresentationclassifierlongtailed} & 47.30 \\
\midrule
\multicolumn{2}{l}{\textit{Normalization-based}} \\
$\tau$-normalize~\cite{kang2020decouplingrepresentationclassifierlongtailed} & 49.40 \\ 
\bottomrule
\end{tabular}
\end{table}

% ---------------- iNaturalist Table ----------------
\begin{table}[H]
\centering
\caption{Top-1 accuracy (\%) on iNaturalist2018 using ResNet-50.}
\label{tab:inat_results}
\begin{tabular}{lcccc}
\toprule
Method & Many & Medium & Few & Overall \\
\midrule
\multicolumn{4}{l}{\textit{Standard / Re-weighting}} \\
Cross-Entropy~\cite{mao2023crossentropylossfunctionstheoretical} & 72.20 & 63.00 & 57.20 & 61.70 \\
CB-Focal~\cite{lin2018focallossdenseobject} & -- & -- & -- & 61.10 \\
Ours + CE & \textbf{72.68} & \textbf{63.91} & \textbf{58.12} & \textbf{64.90} \\
\midrule
\multicolumn{4}{l}{\textit{Margin-based}} \\
LDAM-DRW~\cite{cao2019learningimbalanceddatasetslabeldistributionaware} & -- & -- & -- & 64.60 \\
\midrule
\multicolumn{4}{l}{\textit{Logit-Adjustment}} \\
Logit Adjust~\cite{menon2021longtaillearninglogitadjustment} & -- & -- & -- & 66.40 \\
Balanced Softmax~\cite{ren2020balancedmetasoftmaxlongtailedvisual} & -- & -- & -- & 69.80 \\
ALA~\cite{zhang2024gradientawarelogitadjustmentloss} & 71.30 & 70.80 & 70.40 & 70.70 \\
Ours + LA & \textbf{71.37} & \textbf{69.25} & \textbf{68.10} & \textbf{69.57} \\
Ours + BS & \textbf{71.60} & \textbf{69.91} & \textbf{68.81} & \textbf{70.10} \\
\midrule
\multicolumn{4}{l}{\textit{Decoupled / Two-stage}} \\
cRT~\cite{kang2020decouplingrepresentationclassifierlongtailed} & 69.00 & 66.00 & 63.20 & 65.20 \\
LWS~\cite{kang2020decouplingrepresentationclassifierlongtailed} & 65.00 & 66.30 & 65.50 & 65.90 \\
BBN~\cite{zhou2020bbnbilateralbranchnetworkcumulative} & 49.40 & 70.80 & 65.30 & 66.30 \\
\bottomrule
\end{tabular}
\end{table}

The proposed approach systematically beats the cross-entropy baseline
across all the frequency splits. The most dramatic improvements occur in the \emph{Few} and \emph{Medium} categories, validating the effectiveness of the proposed tail-focused gradient amplification strategy. show the better results achieved by the proposed method.
\begin{figure}[H]
    \centering
    \includegraphics[width=0.9\linewidth]{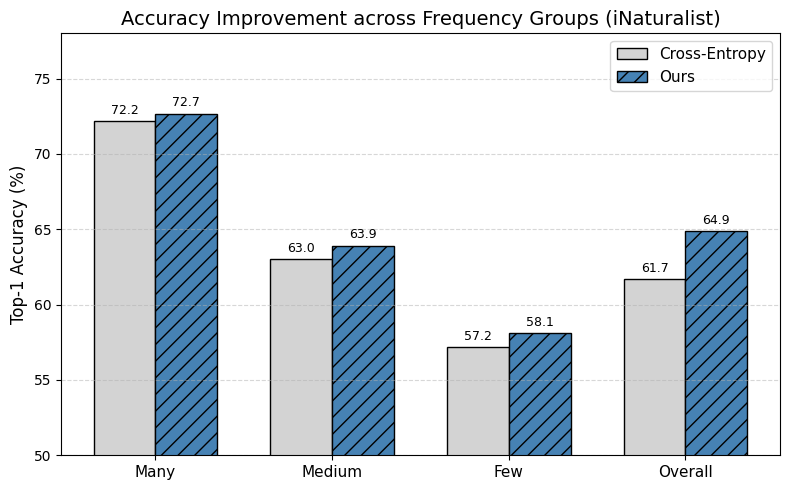}
    \caption{\textbf{Performance by Frequency Group.}
   }
    \label{fig:acc_breakdown}
\end{figure}

\subsection{Ablation Study}
We investigate the sensitivity of the confidence pivot hyperparameter
$\omega$ on large-scale, naturally long-tailed datasets.
To isolate the effect of the proposed confidence-based modulation, we use
standard Cross-Entropy (CE) as the base loss throughout this ablation.
Experiments are conducted on ImageNet-LT and iNaturalist2018 following
their standard training protocols, with all other settings kept fixed.
The resulting Top-1 accuracies are reported in
Table~\ref{tab:omega_ablation_large}.

\begin{table}[H]
\centering
\caption{Ablation study on the suppression factor $\omega$.
Top-1 accuracy (\%) is reported using ResNet-50.}
\label{tab:omega_ablation_large}
\begin{tabular}{lcc}
\toprule
$\omega$ & ImageNet-LT & iNaturalist2018 \\
\midrule
0.25 & 42.68 & 60.81 \\
0.50 & 43.90 & 61.06 \\
\textbf{0.75} & \textbf{44.12} & \textbf{62.90} \\
1.00 & 44.02 & 62.38 \\
\bottomrule
\end{tabular}
\end{table}

As shown in Table~\ref{tab:omega_ablation_large}, setting $\omega = 0.75$
consistently yields the best performance across both datasets.
This choice strikes an effective balance: it allows sufficient gradient
amplification for uncertain tail-class samples ($p_t < 0.75$), while
triggering suppression for highly confident predictions
($p_t \ge 0.75$) that are predominantly associated with head classes.

An overly aggressive suppression threshold ($\omega = 0.25$) leads to
noticeable performance degradation (e.g., 42.68\% on ImageNet-LT), likely
because informative samples are down-weighted prematurely during
training. Conversely, setting $\omega = 1.0$ results in insufficient
suppression of easy examples, yielding diminishing returns.
Observing the consistent trends across datasets of varying scale and
imbalance, we fix $\omega = 0.75$ in all main experiments.

\section{Conclusion}
In this paper, we have proposed a class--confidence aware reweighting strategy
for the purpose of long-tailed learning that jointly incorporates prediction confidence
and class frequency directly at the loss level.
The proposed formulation adaptively focuses the  less confidence examples from
under-represented classes while gently suppressing the over-confident
head-class gradients, without changing logits, margins, or inference
behavior.
Our theoretical work illustrates that the reweighting process
helps to preserve the gradient boundedness and stable optimization dynamics.
Upon extensive experiments conducted on conventional long-tailed benchmarks, there is evidence of clear accuracy improvements, especially in the presence of severe class imbalance, and confirm
that the proposed approach serves as an effective and lightweight
complement to existing logit-adjustment methods.

\bibliography{sources}
\bibliographystyle{ieeetr}

\end{document}